\documentclass{isipta2025}
\usepackage{verbatim}
\usepackage{tikz} 
\usepackage{booktabs} 
\usepackage{csquotes}
\usepackage{array} 
\usepackage{tabularx}

\DeclareMathOperator*{\argmax}{arg\,max}
\DeclareMathOperator*{\argmin}{arg\,min}

\title{Towards conservative inference in credal networks using belief functions: the case of credal chains} 

\author[1]{Marco Sangalli}
\author[1]{Thomas Krak}
\author[1]{Cassio De Campos}

\affil[1]{Uncertainty in AI, Eindhoven University of Technology, The Netherlands}

\begin{document}
\maketitle

\begin{abstract}
    This paper explores belief inference in credal networks using Dempster-Shafer theory. By building on previous work, 
    we propose a novel framework for propagating uncertainty through a subclass of credal networks, namely chains. 
    The proposed approach efficiently yields conservative intervals through belief and plausibility functions, combining computational speed with robust uncertainty representation.
    Key contributions include formalizing belief-based inference methods and comparing belief-based inference against classical sensitivity analysis. Numerical results highlight the advantages and limitations of applying belief inference within this framework, providing insights into its practical utility for chains and for credal networks in general.
\end{abstract}


\section{Introduction}
\label{sec:intro}
Bayesian networks (BNs), as probabilistic graphical models, provide a powerful framework for modelling probabilistic relationships among variables. 
Credal networks (CNs) generalize BNs by allowing for partially specified probabilistic information. It is well-known that inference in both BNs and CNs is in general computationally hard~\cite{MAUA2020133,darwiche2014}. 

We consider the problem of computing inferences in CNs under complete independence \cite{MAUA2020133,debock2016rcn}. 
There are, broadly speaking, two ways by which we might approach this problem. We might either try to identify specific inference problems and/or graph topologies that circumvent the hardness results; or, we can try to develop approximate methods. The first path is arguably difficult, since NP-hardness was already established for marginal inference on relatively simple graph structures like polytrees~\cite{Cassio2005} and even trees with evidence~\cite{maua2014}. Here we consider the second approach.
In particular, we investigate the merit in leveraging the existing machinery from the theory of belief functions, also known as Dempster-Shafer (D-S) theory \cite{dempster1968,Shafer1976}, to approximately compute inferences for a given CN under complete independence. At first glance, this has several upsides. For one, we can fully use the existing computational and algorithmic tools from that theory, which has recently been investigated specifically in the context of Bayesian networks generalized to belief functions~\cite{Shenoy2023}, where the computational complexity of BN inferences is largely reduced. Moreover, a recent result~\cite{VanCamp2025} essentially implies that the global model that we obtain from this is guaranteed to be a conservative bound on the CN in which we are actually interested (even though this result does not hold for marginals, as we will see later).

We work with the subcase of CNs that have the structure of a chain. We require that local models in the CN can be translated (possibly by conservative approximation) to belief functions. These choices are enough to unveil a number of facts about the idea of using D-S theory in inferences for CNs. We investigate different ways in which local models can be made into belief functions. In particular, we are interested in the quality of this approach; i.e. we would like to know (and reduce) how much information is lost with respect to the exact inference in the CN chains in which the local uncertainty models are specified using probability intervals.

Since exact inference in such CNs is computationally tractable, we are able to directly compare these methods from D-S theory to the gold standard of the underlying model of interest and to reach conclusions about the quality of the approximations in chains and beyond. 
Despite optimizations that generate tighter bounds than the direct application of D-S theory, we conclude that---outside specific settings---such belief-based inference often generates too wide intervals to be a practically suitable approximation for the CN inferences of interest. 

\section{Preliminaries}\label{sec:preli}
This section introduces the basic concepts of credal networks~\cite{itip:pgms}, Dempster-Shafer theory~\cite{dempster1967, Shafer1976} of belief functions, and other required concepts for this paper.

\subsection{Bayesian and credal networks} \label{sec:credalnet}
A Bayesian network (BN) is a probabilistic graphical model where a directed acyclic graph (DAG) encodes probabilistic relationships between random variables. Each node represents a variable, and edges indicate conditional dependencies. 
Credal networks (CNs) extend BNs by addressing uncertainty in probability specifications. Instead of fixed point estimates, CNs utilize sets of probability distributions (credal sets), allowing representation of multiple compatible probabilistic models. 
This approach provides a more comprehensive framework for reasoning under incomplete or ambiguous information, capturing broader uncertainty in probabilistic modelling through credal sets rather than single distributions.

A basic example of a CN is one with the DAG structure $A\to B$, where $A:\Omega \to E_A$ and $B:\Omega \to E_B$ are two discrete random variables, $E_A:=\{a_1,\dots, a_n\}$, $E_B:=\{b_1,\dots, b_m\}$, and whose credal sets are specified using collections of intervals:
\begin{align*}
    &\underline{p}^A_i\le \mathbb{P}(A=a_i)\le \overline{p}^A_i,\quad
    \underline{p}^i_j\le \mathbb{P}(B=b_j|A=a_i)\le \overline{p}^i_j.
\end{align*}
Throughout, we consider CNs under \emph{complete independence}~\cite{MAUA2020133,debock2016rcn}, which essentially means that the joint model that we consider can be expressed as the set of all BNs whose local models---$\mathbb{P}(A)$ and $\mathbb{P}(B\vert A)$ in this example---are compatible with these credal sets.

Solving an inference problem for such a CN amounts to computing (tight) lower and upper bounds on some probabilistic quantity with respect to all the BNs in this set. For instance, in this example we may want to compute lower and upper bounds for $\mathbb{P}(B=b_j)$, which are given by
\begin{align*}
    &\underline{\mathbb{P}}(B=b_j):=\min_{\mathbb{P}\in\text{CN}} \sum_{i=1}^n \mathbb{P}(B=b_j | A=a_i) \mathbb{P}(A=a_i),\text{ and} \\
    &\overline{\mathbb{P}}(B=b_j):=\max_{\mathbb{P}\in\text{CN}} \sum_{i=1}^n \mathbb{P}(B=b_j | A=a_i) \mathbb{P}(A=a_i).
\end{align*}

In a chain $X_1\to\dots\to X_k$ we may do something similar to infer bounds on $\mathbb{P}(X_k=\tilde{x})$, using
\begin{align*}
    \underline{\mathbb{P}}(X_k=\tilde{x}):=\min_{\mathbb{P}\in\text{CN}} \sum_{x_1,\dots,x_{k-1}} \mathbb{P}(x_1)\mathbb{P}(x_2|x_1)\dots \mathbb{P}(\tilde{x}|x_{k-1}),\\
    \overline{\mathbb{P}}(X_k=\tilde{x}):=\max_{\mathbb{P}\in\text{CN}} \sum_{x_1,\dots,x_{k-1}} \mathbb{P}(x_1)\mathbb{P}(x_2|x_1)\dots \mathbb{P}(\tilde{x}|x_{k-1}),
\end{align*}
where the minimization and maximization are understood to be taken with respect to all the combinations of the local models specified in the CN.

\subsection{Mass functions and belief functions}
Let $(E,2^{E})$ be a discrete measurable space, where $2^{E}$ represents the power set of $E$. A mass function \cite{itip:specialcases, Shenoy2023} (also called basic probability assignment) $m$ on $E$ is a function $m: 2^{E} \to [0,1]$ such that
\begin{itemize}
    \item $m(\emptyset)=0$ 
    \item $\sum\limits_{V\in 2^{E}} m(V)=1$.
\end{itemize}
Each set $V$ such that $m(V)>0$ is called a \textit{focal set} of $m$. If the only focal set of $m$ is $E$, the mass function is said to be \textit{vacuous}. Meanwhile, if the only set with mass is not the whole space, $m$ is said to be \textit{deterministic}. If $m$ gives masses only to singletons, it is said to be \textit{Bayesian}; there exists a 1-1 correspondence between Bayesian basic probability assignments and discrete probability distributions.

Given a mass function $m$ on $E$, its associated belief function $Bel_m: 2^{E} \to [0,1]$ is defined as 
\begin{equation}\label{eq:mass_to_belief}
    Bel_m(W):=\sum\limits_{V\subseteq W} m(V)\: \text{ for all } W \in 2^E.
\end{equation}
The plausibility function $Pl_m: 2^{E} \to [0,1]$ associated to the mass function $m$ is defined as
\begin{equation}\label{eq:mass_to_plausibility}
    Pl_m(W):=\sum\limits_{V\cap W\ne \emptyset} m(V) \: \text{ for all } W \in 2^E.
\end{equation}
These functions satisfy the conjugacy relation
\begin{equation*}
   Pl_m(W)= 1 - Bel_m(E\smallsetminus W) \: \text{ for all $W\in 2^E$.}
\end{equation*}

\paragraph{Operations with mass functions}
There are three important operations from D-S theory that we will use in this work: marginalization, vacuous extension, and Dempster's rule of combination.

Let $(E,2^{E})$ and $(F,2^{F})$ be two measurable spaces and consider $Y\subseteq E\times F$. We define the projection of $Y$ on $E$ as
\begin{equation*}
    Y^{\downarrow E}:=\{e\in E : (e,f)\in Y\}\,.
\end{equation*}
Let $m$ be a mass function on $E\times F$. The \textit{marginalization} of $m$ to $E$ is a mass function $m^{\downarrow E}$ on $E$ defined as~\cite{Shenoy2023} 
\begin{equation}\label{def: marginal}
    m^{\downarrow E}(V):=\sum_{\substack{Y\subseteq E\times F\\Y^{\downarrow E}=V}} m(Y) \quad\text{for all $V\in 2^{E}$}.
\end{equation}

Let now $m$ be a mass function on $E$. The \textit{vacuous extension} $m^{\uparrow E\times F}$ of $m$ to $E\times F$ is defined as
\begin{equation}
    m^{\uparrow E\times F}(V\times F):=m(V) \quad\text{for all $V\in 2^{E}$}.
\end{equation}

Finally, let $m_1$ and $m_2$ be two mass functions on $(E,2^{E})$. Dempster's rule of combination, denoted by $\otimes$, is an operation between two mass functions that yields a new mass function on the same space. It is defined as
\begin{equation}
    m_1\otimes m_2 (V):=\frac{1}{1-k}\sum_{\substack{V_1,V_2\in 2^{E}\\ V_1\cap V_2=V}} m_1(V_1)m_2(V_2)
\end{equation}
where $k\in [0,1]$ is the \emph{degree of conflict}, defined as
\begin{equation}
    k:=\sum_{\substack{V_1,V_2\in 2^{E}\\ V_1\cap V_2=\emptyset}}m_1(V_1)m_2(V_2).
\end{equation}
If $k=0$ we say that there is no conflict between $m_1$ and $m_2$; conversely if $k=1$ they are in total conflict and they cannot be combined. We shall also note that Dempster's rule of combination is both associative and commutative and, by vacuous extending, can combine mass functions defined on different state spaces.

\subsection{Coherence of probability intervals}\label{subsec:coherent_interval}
Let $(\Omega, \mathcal{A}, \mathbb{P})$ be a probability space 
and let $X:\Omega \to  (E_X,2^{E_X})$ be a discrete random variable taking values in $E_X:=\{x_1,\dots,x_n\}$. Let $(\underline{p}^X,\overline{p}^X):=\{(\underline{p}_i^X,\overline{p}_i^X)\}_{i=1}^n$ be a probability interval for the random variable $X$, namely
\begin{equation*}
    \underline{p}_i^X\le \mathbb{P}(X=x_i) \le \overline{p}_i^X, \quad \text{for all $i=1,\dots,n$}.
\end{equation*}
We say that $(\underline{p}^X,\overline{p}^X)$ is a \textit{coherent} probability interval~\cite{itip:specialcases, MoralGarcia2021} if it satisfies the following conditions:
\begin{description}
    \item[\textbf{[Coh1]}] $\underline{S}_X:=\sum_{i=1}^n \underline{p}_i^X \le 1$ and $\overline{S}_X:=\sum_{i=1}^n \overline{p}_i^X \ge 1$;\label{condcoh1}
    \item[\textbf{[Coh2]}] $\underline{p}_i^X \ge 1-\sum_{h\ne i} \overline{p}_h^X$, for all $i=1,\dots,n$;\label{condcoh2}
    \item[\textbf{[Coh3]}] $\overline{p}_i^X \le 1-\sum_{h\ne i} \underline{p}_h^X$, for all $i=1,\dots,n$.\label{condcoh3}
\end{description}

Given a mass function $m_X$ on $E_X$, we can generate a coherent probability interval $(\underline{p}^X,\overline{p}^X)$ for the random variable $X:\Omega \to  (E_X,2^{E_X})$ using the belief and plausibility functions:
\begin{equation*}
    \underline{p}_i^X=Bel_{m_X}(\{x_i\})\text{, and }\overline{p}_i^X=Pl_{m_X}(\{x_i\}), 
\end{equation*}
for all $i=1,\dots, n$. It immediately follows from the definitions of belief and plausibility that $Bel_{m_X}(\{x_i\})\le Pl_{m_X}(\{x_i\})$ and the three conditions of coherence hold.

So, from a given mass function, we can always generate a coherent probability interval that agrees with it. However, the converse is not true; there are coherent probability intervals for which it is not possible to find a corresponding mass function \cite{MoralGarcia2021}. 
In particular, one \emph{can} always take the natural extension of a coherent probability interval, to obtain a coherent lower prevision~\cite{itip:specialcases}; see e.g.~\eqref{def:belmob} further on. However, this lower prevision need not be a belief function, and hence need not have an associated (non-negative) mass function~\cite{itip:specialcases, MoralGarcia2021}.
Moreover, even if such a mass function does exist, it need not be unique; for a given coherent probability interval, there may in general be multiple mass functions whose associated belief and plausibility functions on the singletons, agree with this interval.
This issue will be revisited in more detail in Section~\ref{sec:good}.

\section{Belief inference in chains}\label{sec:body}
In this section we expand the framework for belief inference  previously introduced by Shenoy~\cite{Shenoy2023}. 
Let $A \rightarrow B$ be a credal network, where $A$ and $B$ are random variables of cardinality $n$ and $m$ respectively. Let us call $E_A=\{a_1,\dots,a_n\}$ and $E_B=\{b_1,\dots,b_m\}$ the state spaces of $A$ and $B$ respectively. Suppose we are given a mass function $m_A$ on $E_A$, and $n$ mass functions $m_{B_{a_i}}$ on $E_B$, that correspond each to a probability interval for $\mathbb{P}(B|A=a_i)$, i.e. 
\begin{equation}\label{eq:basic_conditional_bounds}
\underline{p}_j^i:=Bel_{m_{B_{a_i}}}(b_j)\le \mathbb{P}(b_j|a_i)\le Pl_{m_{B_{a_i}}}(b_j)=:\overline{p}_j^i
\end{equation}
where we simplify the notation  $\mathbb{P}(B=b_j|A=a_i)=\mathbb{P}(b_j|a_i)$.
We can extend each mass function $m_{B_{a_i}}$ to a mass function $m_{B|a_i}$ on $E_A\times E_B$, using Smets' conditional embedding~\cite{Shenoy2023,smets1978modele}, by defining, for all $V\in 2^{E_B}$, the focal sets
\begin{align*}
    m_{B|a_i}(a_i\times V \cup (E_A\smallsetminus \{a_i\} \times E_B))=m_{B_{a_i}}(V).  
\end{align*}
The mass functions built in this way are the only ones on $E_A\times E_B$ that have the following properties:
\begin{itemize}
    \item $(m_{B|a_i}\otimes m_{A=a_i}^{\uparrow E_A\times E_B})^{\downarrow E_B}=m_{B_{a_i}}$ for all $i=1,\dots, n$, where $m_{A=a_i}$ is the deterministic mass function on $E_A$ with $m_{A=a_i}(\{a_i\})=1$;
    \item Each $m_{B|a_i}^{\downarrow E_A}$ is vacuous on $E_A$;
    \item $m_{B|A}:= \bigotimes_{i=1}^n m_{B|a_i}$ is a mass function on $E_A\times E_B$, in which all of the pairwise combinations have no conflict, and $m_{B|A}^{\downarrow E_A}$ is vacuous on $E_A$.
\end{itemize}
In summary, we can express the local uncertainty models of this CN using $m_A$ and $m_{B\vert A}$. The associated global model $m_{A,B}$ is then simply obtained by defining $m_{A,B}=m_A^{\uparrow E_A \times E_B}\otimes m_{B|A}$. It can be easily checked that the mass functions that are combined to obtain $m_{A,B}$ have no conflict, since the intersection of any two focal sets is always non-empty.

From this global model, we can derive any inferences of interest. In this work, we generally focus on marginal inference, e.g. the marginal uncertainty on $B$, expressed as $m_B=m_{A,B}^{\downarrow B}$.
 Using the mass function $m_B$, we may generate a probability interval on $B$, $(\underline{p}^B,\overline{p}^B)$, using $Bel_{m_B}$ and $Pl_{m_B}$. Our first result is a closed-form expression for the values of this probability interval.
\begin{theorem}\label{th:inferenceintervals}
Let $m_A$, $m_{B\vert A}$, $m_{A,B}$, and $m_B$ be as defined above. Then for all $j=1,\ldots,m$ it holds that
\begin{align}
    \underline{p}^B_j&:= Bel_{m_B}(\{b_j\}) = \sum\limits_{V\in 2^{E_A}} \left( m_A(V)\prod_{i: a_i\in V} \underline{p}^i_j \right)\,, \text{and} \label{eq:lowboundB}\\
    \overline{p}^B_j&:=Pl_{m_B}(\{b_j\}) = 1-\sum\limits_{V\in 2^{E_A}} \left( m_A(V)\prod_{i: a_i\in V} (1-\overline{p}^i_j) \right)\,,\label{eq:upboundB}
\end{align}
with $\underline{p}_j^i$ and $\overline{p}_j^i$ as in Equation~\eqref{eq:basic_conditional_bounds}.
\end{theorem}
This result is a particular case of Equation (5.1) in \cite{smets1993belief} and the proof can be found in the Appendix.

One may now wonder, how the interval $\smash{(\underline{p}^B, \overline{p}^B)}$ obtained using Theorem~\ref{th:inferenceintervals}, relates to the values $\underline{\mathbb{P}}(B)$ and $\overline{\mathbb{P}}(B)$ computed from a CN under complete independence, as in Section~\ref{sec:credalnet}.
At this stage we need to make the observation that for the CN in Section~\ref{sec:credalnet}, our uncertainty model for $A$ was a probability interval $\smash{(\underline{p}^A, \overline{p}^A)}$. Conversely, our discussion here started from an (arbitrary) mass function $m_A$ and, as mentioned in Section~\ref{subsec:coherent_interval}, there need not be a one-to-one correspondence between the two.

As the following example shows, it is sometimes possible to choose a mass function $m_A$ that agrees with a given $\smash{(\underline{p}^A, \overline{p}^A)}$, such that the values $\smash{(\underline{p}^B, \overline{p}^B)}$ are \emph{not} guaranteed bounds on $\underline{\mathbb{P}}(B)$ and $\overline{\mathbb{P}}(B)$.
\begin{example}\label{ex:counterexbound}
    Consider a CN $A \to B$, where the random variable $A$ has state space $E_A = \{a_1, \ldots, a_4\}$ and satisfies $0.2 \leq \mathbb{P}(A = a_i) \leq 0.3$ for all $i = 1, \dots, 4$. Consider the mass function $m_A$ on $E_A$ that assigns mass $0.2$ to singletons, $m(\{a_1, a_2\}) = m(\{a_3, a_4\}) = 0.1$, and zero mass elsewhere. It is easily verified that indeed $Bel_{m_A}(\{a_i\})=0.2$ and $Pl_{m_A}(\{a_i\})=0.3$ for all $i=1,\ldots,4$.
    
    Fix any $j\in\{1,\dots,|E_B|\}$, and suppose the collection of lower bounds for $\mathbb{P}(b_j|a_i)$ is $(\underline{p}_j^i)_{i=1}^4=(0.2,0.4,0.8,0.9)$.
    If we compute $\smash{\underline{p}_j^B}$ using (\ref{eq:lowboundB})  we obtain a value of $0.54$, whereas it is easy to check that
    \begin{equation*}
        \underline{\mathbb{P}}(B=b_j)=\min\limits_{\substack{\underline{p}^A_i\le \mathbb{P}(a_i)\le \overline{p}^A_i\\ \sum_i \mathbb{P}(a_i) =1}} \sum_{i=1}^n \mathbb{P}(a_i) \underline{p}^i_j=0.52\,.
    \end{equation*}
    This provides an example where the credal lower bound for $\mathbb{P}(B=b_j)$ is lower than $\underline{p}_j^B$.  
    \hfill$\diamond$
\end{example}
We next focus on finding mass functions which ensure that intervals generated by belief inference, provide a conservative approximation of credal intervals.

\subsection{Standard good mass and good probability intervals}\label{sec:good}
We say that a coherent probability interval $(\underline{p}^A,\overline{p}^A)$ for $A:\Omega \to E_A$ is \textit{good} if
\begin{equation}\label{eq:def_good_delta}
    \Delta_A:=\overline{S}_A+(n-2)\underline{S}_A-(n-1)\ge 0\,,
\end{equation}
where $n=|E_A|$, $\overline{S}_A:=\sum_{i=1}^n \overline{p}^A_i$ and $\underline{S}_A:=\sum_{i=1}^n \underline{p}^A_i$.
If a coherent probability interval $\smash{(\underline{p}^A,\overline{p}^A)}$ for the random variable $A$ is \textit{good}, then there always exists a mass function $m_A$ with at most $2n +1$ focal sets, such that the associated belief and plausibility of singletons match this interval. One such mass function $m_A$ that satisfies this property is characterized by
\begin{align*}
    &m_A(\{a_i\})=\underline{p}_i^A \\
    &m_A(E_A\smallsetminus\{a_i\})=1-\overline{p}_i^A-\sum_{h\ne i}\underline{p}_h^A \\
    &m_A(E_A)=\Delta_A
\end{align*}
for all $i=1,\dots,n$.
We shall call this mass function on $E_A$ the \textit{standard good mass} (SGM) associated to $\smash{(\underline{p}^A,\overline{p}^A)}$. 

\paragraph{Relation between M\"obius inverse and standard good mass}
It is well-known that there is a one-to-one correspondence between belief (resp. plausibility) functions, and mass functions. Indeed, Equation~\eqref{eq:mass_to_belief} (resp.~\eqref{eq:mass_to_plausibility}) gives the connection in one direction. To go from a belief function $Bel$ to a mass function $m$, one uses the M\"obius inverse
\begin{equation*}
    m(V)=\sum_{W\subseteq V}(-1)^{|V\smallsetminus W|} Bel(W).
\end{equation*}

If we start from a coherent probability interval $\smash{(\underline{p}^A, \overline{p}^A)}$ and want to use this machinery, we therefore first have to represent it as a belief function. It is also well-known that this is not always possible; see e.g.~\cite{MoralGarcia2021} for a characterization of when it is. Following this work, the basic idea is to take the natural extension of this interval,
\begin{equation} \label{def:belmob}
\begin{split}
    &Bel(V)=\max\left\{{\sum_{i:a_i \in V} \underline{p}_i^A} , 1 - \sum_{i:a_i \notin V}\overline{p}_i^A\right\}\,,\\
    &Pl(V)=\min\left\{{\sum_{i:a_i \in V} \overline{p}_i^A} , 1 - \sum_{i:a_i \notin V}\underline{p}_i^A\right\}.\\
\end{split}
\end{equation}
If (and only if) this extension is indeed a belief (resp. plausibility) function, we can then use the M\"obius inverse to find an associated mass function. 

In the specific case that the interval $\smash{(\underline{p}^A, \overline{p}^A)}$ is \emph{good}, we have the following result.
\begin{theorem}\label{th:mobius}
    If $(\underline{p}^A,\overline{p}^A)$ is a good coherent probability interval for the random variable $A:\Omega \to E_A=\{a_1,\dots,a_n\}$, then the mass function $m_A^M$ obtained using (\ref{def:belmob}) and the M\"obius inverse, coincides with the standard good mass associated to $(\underline{p}^A,\overline{p}^A)$.
\end{theorem}
The proof of this result can be found in the Appendix.

\subsection{Credal bound for standard good masses}\label{subsec:bound}
In this section we present one of the main results of this paper, namely the fact that when the local probability intervals of a CN can be represented using SGMs, then the marginal inference $(\underline{p}^B,\overline{p}^B)$ computed using~\eqref{eq:lowboundB} and~\eqref{eq:upboundB} is always a conservative approximation of the credal bounds $\underline{\mathbb{P}}(B)$ and $\overline{\mathbb{P}}(B)$.
\begin{theorem} \label{th:bound}
    If $(\underline{p}^A, \overline{p}^A)$ is a good coherent probability interval and $m_A$ its associated standard good mass, then for all $j=1,\dots, m$ it holds that
    \begin{align}
        \underline{p}^B_j\le \min_{\mathcal{C}} \sum_{i=1}^n \mathbb{P}(B=b_j|A=a_i)\mathbb{P}(A=a_i)\label{ineq:lowbound}\\ 
        \overline{p}^B_j\ge \max_{\mathcal{C}} \sum_{i=1}^n \mathbb{P}(B=b_j|A=a_i)\mathbb{P}(A=a_i) \label{ineq:upbound}
    \end{align}
    where $\mathcal{C}=\{\underline{p}^i_j\le \mathbb{P}(B=b_j|A=a_i)\le \overline{p}^i_j\text{ , } \underline{p}^A_i\le \mathbb{P}(A=a_i)\le \overline{p}^A_i \text{ , } \sum_{i} \mathbb{P}(A=a_i)=1\}$.
\end{theorem}
\begin{proof}
    We start by proving (\ref{ineq:lowbound}). Fix $j\in \{1,\dots, m\}$. Without loss of generality let us order $\{\underline{p}^i_j\}_{i=1}^n$ increasingly:  
    \begin{equation*}
        \underline{p}^1_j\le\dots\le \underline{p}^n_j 
    \end{equation*}
    We can rewrite the minimum in (\ref{ineq:lowbound}) as:
    \begin{equation*}
        \min_{\mathcal{C}} \sum_{i=1}^n \mathbb{P}(b_j|a_i)\mathbb{P}(a_i) = \min_{\substack{\underline{p}^A_i\le x_i\le \overline{p}^A_i\\ \sum_{i=1}^n x_i=1}} \sum_{i=1}^n \underline{p}^i_j x_i.
    \end{equation*}
    Since by \hyperref[condcoh1]{[Coh1]} it holds that
    \begin{equation*}
        \sum_{i=1}^n \underline{p}^A_i\le 1\le \sum_{i=1}^n \overline{p}^A_i\,,
    \end{equation*} 
    there exists a least index $\Tilde{i}$ such that
    \begin{equation*}
        \overline{p}^A_1+\dots+\overline{p}^A_{\Tilde{i}-1}+x_{\Tilde{i}}+\underline{p}^A_{\Tilde{i}+1}+\dots+ \underline{p}^A_n=1
    \end{equation*}
     where $x_{\Tilde{i}}\in [\underline{p}^A_{\Tilde{i}},\overline{p}^A_{\Tilde{i}}]$ is unambiguously determined.
    In order to achieve the minimum in (\ref{ineq:lowbound}) the values of $x_i$ need to be as high as possible when $\underline{p}^i_j$ is small and vice versa. Therefore, the vector $(x_i)_{i=1}^n$ that realizes the minimum is $(\overline{p}^A_1,\dots,\overline{p}^A_{\Tilde{i}-1},x_{\Tilde{i}},\underline{p}^A_{\Tilde{i}+1},\dots, \underline{p}^A_n)$, where $\Tilde{i}$ is taken as small as possible.

    Let us rewrite inequality (\ref{ineq:lowbound}) using (\ref{eq:lowboundB}) and the definition of standard good mass:
    \begin{align*}
        \sum_{i=1}^n \underline{p}^A_i\underline{p}^i_j + \sum_{i=1}^n m_A(E_A\smallsetminus\{a_i\})\prod_{h\ne i}\underline{p}_j^h +\Delta_A\prod_{i=1}^n\underline{p}_j^i\le\\
        \sum_{i<\Tilde{i}}\overline{p}^A_i\underline{p}^i_j +x_{\Tilde{i}}\underline{p}_j^{\Tilde{i}}+\sum_{i>\Tilde{i}}\underline{p}^A_i\underline{p}^i_j.
    \end{align*}
    Therefore, we need to prove that:
    \begin{align*}
        \sum_{i<\Tilde{i}}(\overline{p}^A_i-\underline{p}^A_i)\underline{p}^i_j + (x_{\Tilde{i}}-\underline{p}^A_{\Tilde{i}})\underline{p}_j^{\Tilde{i}} &\ge \\ \sum_{i=1}^n m_A(E_A\smallsetminus\{a_i\})\prod_{h\ne i}\underline{p}_j^h +\Delta_A\prod_{i=1}^n\underline{p}_j^i &=: \star.
    \end{align*}
    Let us note that due to the condition of coherency \hyperref[condcoh2]{[Coh2]}, it follows that $\Tilde{i}\ge 2$ for every $n\ge 3$. Using that $\underline{p}^1_j\ge \prod_{h\ne i}\underline{p}^h_j$ for all $i\ne 1$ and $\underline{p}^2_j\ge \prod_{h\ne 1}\underline{p}^h_j$, we get
    \begin{equation*}
        \star \le \underline{p}^1_j\left(\sum_{i\ne 1} m_A(E_A\smallsetminus\{a_i\}) + m_A(E_A)\right)+\underline{p}^2_j m_A(E_A\smallsetminus\{a_1\})
    \end{equation*}
    Let us observe that it follows from the definition of standard good mass that
    \begin{equation*}
        \sum_{i\ne 1} m_A(E_A\smallsetminus\{a_i\}) + m_A(E_A) =\overline{p}^A_1-\underline{p}^A_1.
    \end{equation*}
    Therefore, it is enough to prove that
    \begin{equation*}
        \sum_{2\le i<\Tilde{i}}(\overline{p}^A_i-\underline{p}^A_i)\underline{p}^i_j + (x_{\Tilde{i}}-\underline{p}^A_{\Tilde{i}})\underline{p}_j^{\Tilde{i}}\ge \underline{p}^2_j m_A(E_A\smallsetminus\{a_1\})
    \end{equation*}
    and this inequality immediately follows from the fact that
    \begin{align*}
        \sum_{2\le i<\Tilde{i}}(\overline{p}^A_i-\underline{p}^A_i) +  (x_{\Tilde{i}}-\underline{p}^A_{\Tilde{i}})&=1-\underline{S}_A-(\overline{p}^A_1-\underline{p}^A_1)\\&=m_A(E_A\smallsetminus\{a_1\}).
    \end{align*}
    
    The proof for (\ref{ineq:upbound}) follows a procedure similar to the one just presented. 
    It is also easy to check that (\ref{ineq:lowbound}) and (\ref{ineq:upbound}) also hold for the particular case $n=2$.
\end{proof}
This result is straightforwardly extended to chains of arbitrary length. Indeed, if we consider the chain $A\to B\to C$ where $E_A=\{a_1,\dots,a_n\}, E_B=\{b_1,\dots, b_m\}, E_C=\{c_1,\dots, c_l\}$, and 
\begin{align*}
    &\underline{p}^A_i\le \mathbb{P}(A=a_i)\le \overline{p}^A_i\\
    &\underline{p}^i_j\le \mathbb{P}(B=b_j|A=a_i)\le \overline{p}^i_j \\
    &\underline{p}^j_k\le \mathbb{P}(C=c_k|B=b_j)\le \overline{p}^j_k,\nonumber
\end{align*}
it also holds for all $k=1,\dots, l$ that 
\begin{equation}
    \underline{p}^C_k\le \min \sum_{i,j}\mathbb{P}(a_i)\mathbb{P} (b_j|a_i)\mathbb{P}(c_k|b_j)\,,  \label{eqglobalteo}
\end{equation}
and similarly for the upper bound $\overline{p}^C_k$. The quantity $\underline{p}^C_k$ is obtained by applying~\eqref{eq:lowboundB},~\eqref{eq:upboundB} twice---first to obtain $\smash{(\underline{p}^B,\overline{p}^B)}$, then to combine \emph{that} with the $\smash{(\underline{p}^j_k, \overline{p}^j_k)}$---whence it follows from Theorem~\ref{th:bound} that
\begin{equation*}
    \underline{p}^C_k\le \min_{\substack{\underline{\mathbb{P}}(b_j)\le \mathbb{P}(b_j) \le \overline{\mathbb{P}}(b_j)\\ \sum_j \mathbb{P}(b_j)=1 }} \sum_j \mathbb{P}(b_j)\underline{p}^j_k=:\diamond
\end{equation*}
with $\underline{\mathbb{P}}(b_j)$ and $\overline{\mathbb{P}}(b_j)$ as before. Moreover, we clearly also have that $\diamond \leq \min \sum_{i,j}\mathbb{P}(a_i)\mathbb{P} (b_j|a_i)\underline{p}^j_k$,
which is all that is required to obtain~\eqref{eqglobalteo}.
A similar argument can be applied to prove that the bound holds for $\overline{p}^C_k$ and for chains of arbitrary length. 

\subsection{Fixing bad probability intervals} \label{par:fixingintervals}
Suppose we are considering a CN $X_1\rightarrow \dots \rightarrow X_k$ and suppose we are given good probability intervals for the random variable $X_1$ and for each conditional distribution $X_l|(X_{l-1}=x)$ of every step of the chain. Then, we may apply formulas (\ref{eq:lowboundB}) and (\ref{eq:upboundB}) step-by-step starting from $X_1$ to eventually infer a probability interval for $X_k$, $(\underline{p}^{X_k},\overline{p}^{X_k})$. At each step, after applying \eqref{eq:lowboundB} and \eqref{eq:upboundB}, we revert the probability interval into a mass function and proceed along the chain. We call this process \textit{belief inference} on the credal network $X_1\to\dots\to X_k$.
This probability interval on $X_k$  is the same as the one associated with the mass function 
\begin{equation}
 \label{eq:fullmass}
(m_{X_1}\otimes m_{X_2|X_1}\otimes \dots \otimes m_{X_k|X_{k-1}})^{\downarrow X_k},
\end{equation}
thanks to the local computation property \cite{Shenoy2023} of mass functions, which states that
\begin{align*}
    (m_{X_1}\otimes m_{X_2|X_1}\otimes m_{X_3|X_2})^{\downarrow (X_2,X_3)}=\\(m_{X_1}\otimes m_{X_2|X_1})^{\downarrow X_2}\otimes m_{X_3|X_2}.
\end{align*}
Directly computing \eqref{eq:fullmass} is computationally expensive as $m_{X_1,\dots, X_k}$ gives mass to an exponential number of focal sets namely $\mathcal{O}(2^{kn})$, where $|X_i|=n$ for all $i=1,\dots, k$. 

The step-by-step process we described requires only $\mathcal{O}(kn^3)$ operations when employing standard good masses. However, this method is not always applicable since there is no guarantee that, after one step of inference, we can find a mass function associated to the inferred probability interval. 
One of the possible solutions to this issue is to consider the closest good outer approximation to a bad probability interval. 

If we are considering the credal network $A\to B$ and we are given a bad probability interval for the random variable $A$, a simple way to fix $(\underline{p}^A,\overline{p}^A)$ consists in slightly enlarging the intervals, by distributing the quantity $-\Delta_A = (n-1) -\overline{S}_A - (n-2)\underline{S}_A>0$ over the $\{\overline{p}^A_i\}_{i=1}^n$; it follows from Equation~\eqref{eq:def_good_delta} that the resulting interval will be good.
The bottom line of this approach is that we are bartering precision for computational speed and tractability. 

Fixing a bad probability interval for $A$ does not break inequalities (\ref{ineq:lowbound}) and (\ref{ineq:upbound}), as the following corollary states.
\begin{corollary}
    If $(\underline{p}^A, \overline{p}^A)$ is a bad probability interval and $m_A$ is the standard good mass associated to an arbitrary fixing of $(\underline{p}^A, \overline{p}^A)$, then inequalities (\ref{ineq:lowbound}) and (\ref{ineq:upbound}) still hold.
\end{corollary}
The proof of this result is straightforward and omitted due to page limit constraints. 

If some probability interval for the conditional random distributions $B\mid(A=a_i)$ are bad, we may compute \(\underline{p}_j^B\) and \(\overline{p}_j^B\) using expressions (\ref{eq:lowboundB}) and (\ref{eq:upboundB}), treating the bad probability intervals as if they were good. This occurs because increasing the upper bounds \(\{\overline{p}^i_j\}_{j=1}^m\) to fix \((\underline{p}^i, \overline{p}^i)\) does not affect \(\underline{p}_j^B\) in expression (\ref{eq:lowboundB}). When calculating $\overline{p}^B_j$ for a fix $j$, there always exists a fixing for every $(\underline{p}^i, \overline{p}^i)$ such that $\overline{p}^i_j$ remains unchanged for all $i=1,\dots, n$. In order to prove this statement, it is enough to show that we can always distribute $-\Delta_i=(m-1)-\overline{S}_i-(m-2)\underline{S}_i$ over $\{\overline{p}^i_k\}_{k\ne j}$. Each $\overline{p}^i_k$ can increase by at most $r^i_k=1-\underline{S}_i-(\overline{p}^i_k - \underline{p}^i_k)$ to preserve the condition of coherency \hyperref[condcoh3]{[Coh3]}. Therefore
\begin{align*}
    \sum_{k\ne j} r^i_k&=\sum_{k\ne j} (1-\underline{S}_i-\overline{p}^i_k+\underline{p}_k^i)\\
    &=(m-1)-(m-1)\underline{S}_i-\overline{S}_i+\overline{p}^i_j+\underline{S}_i-\underline{p}^i_j\\
    &=(m-1)-\overline{S}_i-(m-2)\underline{S}_i+\overline{p}^i_j-\overline{p}^i_j\\
    &=-\Delta_i+\overline{p}^i_j-\overline{p}^i_j\ge -\Delta_i.
\end{align*}
This confirms that if we are provided with any coherent probability interval for \(B \mid (A = a_i)\) for all \(i\), inference can be performed without fixing these intervals into good ones.

\paragraph{Ad-hoc fixing} 
Consider the credal network $A\to B$ where we are given a bad probability interval for the random variable $A$, $(\underline{p}^A,\overline{p}^A)$, and coherent probability intervals for the conditional distributions, namely $\{(\underline{p}^i,\overline{p}^i)\}_{i=1}^n$. 
Since inequality (\ref{ineq:lowbound}) holds for every fixing of $(\underline{p}^A,\overline{p}^A)$, we shall consider the distribution of $-\Delta_A$ that maximizes the value of $\underline{p}_j^B$ thus minimizing the width of the inferenced interval. Let \( t = \{t_i\}_{i=1}^n \) be the values added to \(\{\overline{p}^A_i\}_{i=1}^n\) to transform the bad probability interval into a good one. It necessarily holds that $0\le t_i\le r_i^A$ and $\sum_i t_i=-\Delta_A$. Let $\underline{t}^{*j}$ be the distribution of such values that maximizes $\underline{p}_j^B$, namely
\begin{align}
    \underline{t}^{*j}&= \argmax_{\substack{0\le t_i\le r_i^A\\ \sum_i t_i=-\Delta_A}} \sum_{i=1}^n \left( \Bigl(1-(\overline{p}^A_i +t_i)-\sum_{h\ne i}\underline{p}^A_h\Bigr)\prod_{h\ne i}\underline{p}^h_j\right)\nonumber\\
    &=\argmin_{\substack{0\le t_i\le r_i^A\\ \sum_i t_i=-\Delta_A}} \sum_{i=1}^n \left(t_i\prod_{h\ne i} \underline{p}^h_j\right).\label{eq:tstatrlow}
\end{align}
  Similarly, we shall define $\overline{t}^{*j}$, the distribution over $\{\overline{p}^A_i\}_{i=1}^n$ that minimizes $\overline{p}^B_j$.
\begin{equation}
    \overline{t}^{*j}=\argmin_{\substack{0\le t_i\le r_i^A\\ \sum_i t_i=-\Delta_A}} \sum_{i=1}^n \left(t_i\prod_{h\ne i} (1-\overline{p}^h_j)\right).
\end{equation}
By first fixing $\smash{(\underline{p}^A, \overline{p}^A)}$ with $\underline{t}^{*j}$ to compute $\underline{p}^B_j$, and then fixing $(\underline{p}^A, \overline{p}^A)$ with $\overline{t}^{*j}$ to compute $\overline{p}^B_j$, we obtain the smallest interval — among those generated using the standard good mass — for $\mathbb{P}(B = b_j)$. We shall note that inequalities (\ref{ineq:lowbound}) and (\ref{ineq:upbound}) remain valid, as they hold for any possible fixing of the probability interval $(\underline{p}^A, \overline{p}^A)$.
This operation can be repeated for each \(j\), resulting in a probability interval \((\underline{p}^B, \overline{p}^B)\) for \(B\). We refer to this method of fixing \((\underline{p}^A, \overline{p}^A)\) to perform inference as \textit{ad-hoc fixing}.
\begin{example}\label{ex:urn}
    Consider an urn with $N=100$ balls coloured in red (R), green (G), blue (B) or yellow (Y). We do not know the proportion of balls for each colour and we have the following estimates:
    \begin{align*}
        6\le N(R) \le33 &\qquad 10\le N(G) \le42 \\
        15\le N(B) \le32 &\qquad 25\le N(Y) \le 58.
    \end{align*}
    We possess four coins, one for each colour, which we know are unbalanced yet we are still uncertain about the exact mean value:
    \begin{align*}
        0.8\le \mathbb{P}(H|col=R)\le 0.9 &\quad 0.6\le \mathbb{P}(H|col=G)\le 0.7\\
        0.4\le \mathbb{P}(H|col=B)\le 0.5 & \quad 0.2\le \mathbb{P}(H|col=Y)\le 0.3.
    \end{align*}
    We sample uniformly at random a ball from the urn and flip the associated coin. The inference task of calculating the lower probability of getting heads (H) can be computed using belief inference with ad-hoc fixing. The probability interval associated with the random variable $col$ representing the colour of the sampled ball is 
    \begin{align*}
        0.06\le \mathbb{P}(R) \le0.33 &\qquad 0.1\le \mathbb{P}(G) \le0.42 \\
        0.15\le \mathbb{P}(B) \le0.32 &\qquad 0.25\le \mathbb{P}(Y) \le 0.58.
    \end{align*}
    This probability interval is bad ($\Delta= - 0.23$) thus requiring fixing. We can widen this probability interval by at most $r=(0.17,0.12,0.27,0.11)$ while preserving the conditions of coherency. The values of $\prod_{h\ne i} \underline{p}^h_j$ in \eqref{eq:tstatrlow} are $(0.048, 0.064,0.096,0.192)$. 
   Therefore, the best way to fix the probability interval for $col$ in order to maximize $\underline{p}_H$, the lower probability of getting heads, is 
    \( \underline{t}^*_H=(0.17,0.06,0,0)\) and we find 
    \[
    \underline{p}_H \approx 0.253  \le 0.328=\underline{\mathbb{P}}(H)
    \]
     where this last value represents the lower probability obtained via classical sensitivity analysis.
     \hfill$\diamond$
\end{example}
\paragraph{Ad-hoc mass functions} 
Recalling that our goal is to estimate credal intervals as accurately as possible, one potential approach is to explore alternatives beyond relying solely on the standard good mass. Specifically, we could determine the mass function on $A$ that generates the smallest possible probability interval for $B$. For a fixed $j=1,\dots,m$, we may consider the following mass function on $A$:
\begin{equation}
    \underline{m}^{max}_{j}:=\argmax_{\substack{0\le m_A(V)\le 1\\ Bel_{m_A}(\{a_i\})=\underline{p}^A_i \\ Pl_{m_A}(\{a_i\})=\overline{p}^A_i} } \sum\limits_{V\in 2^{E_{A}}} \left( m_{A}(V)\prod_{i: x_i\in V} \underline{p}^i_j \right).
\end{equation}
This mass function yields the largest possible value for \(\underline{p}^B_j\). 
However, as mentioned at the beginning of Section~\ref{sec:body}, there is no guarantee that this value of \(\underline{p}^B_j\) is always less than the lower bound of the corresponding credal interval, given by $\min\limits_{\underline{p}^A_i\le x_i\le \overline{p}^A_i, \sum_i x_i =1} \sum_{i=1}^n \underline{p}^i_j x_i$. Example~\ref{ex:counterexbound} shows one such example. 

Referring to the proof of Theorem~\ref{th:bound}, we observe that the use of the standard good mass ensures that every set $V$ with $|V|\ge 2$, always contains either $\{a_1\}$ or $\{a_2\}$. This property does not hold for an arbitrary mass function on $A$, preventing the use of the inequality $\prod_{i:a_i\in V}\underline{p}^i_j\le \underline{p}^2_j$, which is fundamental to complete the proof of Theorem~\ref{th:bound}. It is important to recall that by $\underline{p}^2_j$ we mean the second smallest value among $\{\underline{p}_j^i\}_{i=1}^n$.

For instance, in Example~\ref{ex:counterexbound}, the set $\{a_3,a_4\}$ is a focal set for $m_A$ and it holds that $\underline{p}^3_j \underline{p}^4_j $ is greater than $\underline{p}^2_j$. 

Therefore, if we assign zero mass to sets $V$ with $|V|\ge 2$ such that $\prod_{i:a_i\in V}\underline{p}^i_j> \underline{p}^2_j$,  inference using this mass function on $A$ always guarantees that $\underline{p}^B_j$ is smaller than the corresponding credal lower bound. 

Let us introduce $\underline{m}^*_j$, a mass function on $A$ designed to ensure that inequality (\ref{ineq:lowbound}) holds:
\begin{equation}\label{eq:mstardown}
    \underline{m}^{*}_{j}:=\argmax_{\mathclap{\substack{0\le m_A(V)\le 1\\ Bel_{m_A}(\{a_i\})=\underline{p}^A_i \\ Pl_{m_A}(\{a_i\})=\overline{p}^A_i \\ m_A(V)=0 \text{ if } |V|\ge 2 \land \prod\limits_{i:a_i\in V}\underline{p}^i_j> \underline{p}^2_j }}} \quad\sum\limits_{V\in 2^{E_{A}}} \left( m_{A}(V)\prod_{i: x_i\in V} \underline{p}^i_j \right).
\end{equation}

Similarly, we may define a mass function that minimizes the upper bound $\overline{p}_j^B$ while ensuring that inequality \eqref{ineq:upbound} holds:
\begin{equation}\label{eq:mstarup}
    \overline{m}^{*}_{j}:=\argmax_{\mathclap{\substack{0\le m_A(V)\le 1\\ Bel_{m_A}(\{a_i\})=\underline{p}^A_i \\ Pl_{m_A}(\{a_i\})=\overline{p}^A_i \\ \qquad m_A(V)=0 \text{ if } |V|\ge 2 \land  \prod\limits_{i:a_i\in V}(1-\overline{p}^i_j)> 1-\overline{p}^2_j }}} \quad\sum\limits_{V\in 2^{E_{A}}} \left( m_{A}(V)\prod_{i: x_i\in V} (1-\overline{p}^i_j) \right).
\end{equation}
We refer to these mass functions on $A$ as the \textit{ad-hoc mass}  functions for computing the interval $(\underline{p}_j^B,\overline{p}^B_j)$.
The operation of inferencing with ad-hoc mass functions can be repeated for each \(j\), resulting in a probability interval \((\underline{p}^B, \overline{p}^B)\) for \(B\). Practically, computing each of the $2m$ ad-hoc mass functions consists of solving a linear programming problem.
\begin{example}
    Let us try to use ad-hoc mass functions to compute the lower probability of getting heads in Example \ref{ex:urn}. Solving the linear programming problem of \eqref{eq:mstardown} we obtain
    \begin{gather*}
        \underline{m}^*_H(\{R\})=0.06 \quad \underline{m}^*_H(\{G\})=0.1 \quad \underline{m}^*_H(\{B\})=0.15\\
       \underline{m}^*_H(\{Y\})=0.25 \quad \underline{m}^*_H(\{R,B\})=0.054\\
       \underline{m}^*_H(\{G,B\})=0.048 \quad \underline{m}^*_H(\{R,G,B\})=0.138\\
       \underline{m}^*_H(\{R,G,Y\})=0.174 \quad \underline{m}^*_H(\{R,G,B,Y\})=0.006.
    \end{gather*}
    The lower probability calculated using the ad-hoc mass function $\underline{m}^*_H$ in \eqref{eq:lowboundB} is $0.2861$, which is a slight improvement over the ad-hoc fixing method.
    \hfill$\diamond$
\end{example}
Let us note that belief inference with ad-hoc mass functions may generate probability intervals that cannot be represented by an ad-hoc mass function, as the argmax in definitions \eqref{eq:mstardown} and \eqref{eq:mstarup}  might be taken over unfeasible constraints. A practical solution to this issue is presented in the following section.
\section{Numerical analysis} \label{sec:vacanalysis}
In Section~\ref{subsec:bound} it was shown that, under specific assumptions, intervals generated using belief inference are more conservative then their credal counterparts. Here, we aim to address the question of how much larger these intervals become. 

Let us present some numerical results about the chain $X_1\to \dots \to X_k$. In order to perform numerical simulations on the chain, we fixed the cardinality of each random variable $|X_i|=n$ and sampled uniformly at random (using a hit\&run algorithm~\cite{SmithHnR1984}) $n(k-1)+1$ probability intervals; one for each conditional distribution $X_l|(X_{l-1}=x_{i}), l=2,\dots,k \text{ and } i=1,\dots,n$, and one for $X_1$. 
Then, we implemented different methods to perform inference, namely credal inference, belief inference with ad-hoc mass functions and belief inference using the ad-hoc fixing method. We compared the results of these simulations using the average width of inferenced intervals as a metric.

\begin{figure}[ht]
    \centering
    \includegraphics[width=\linewidth]{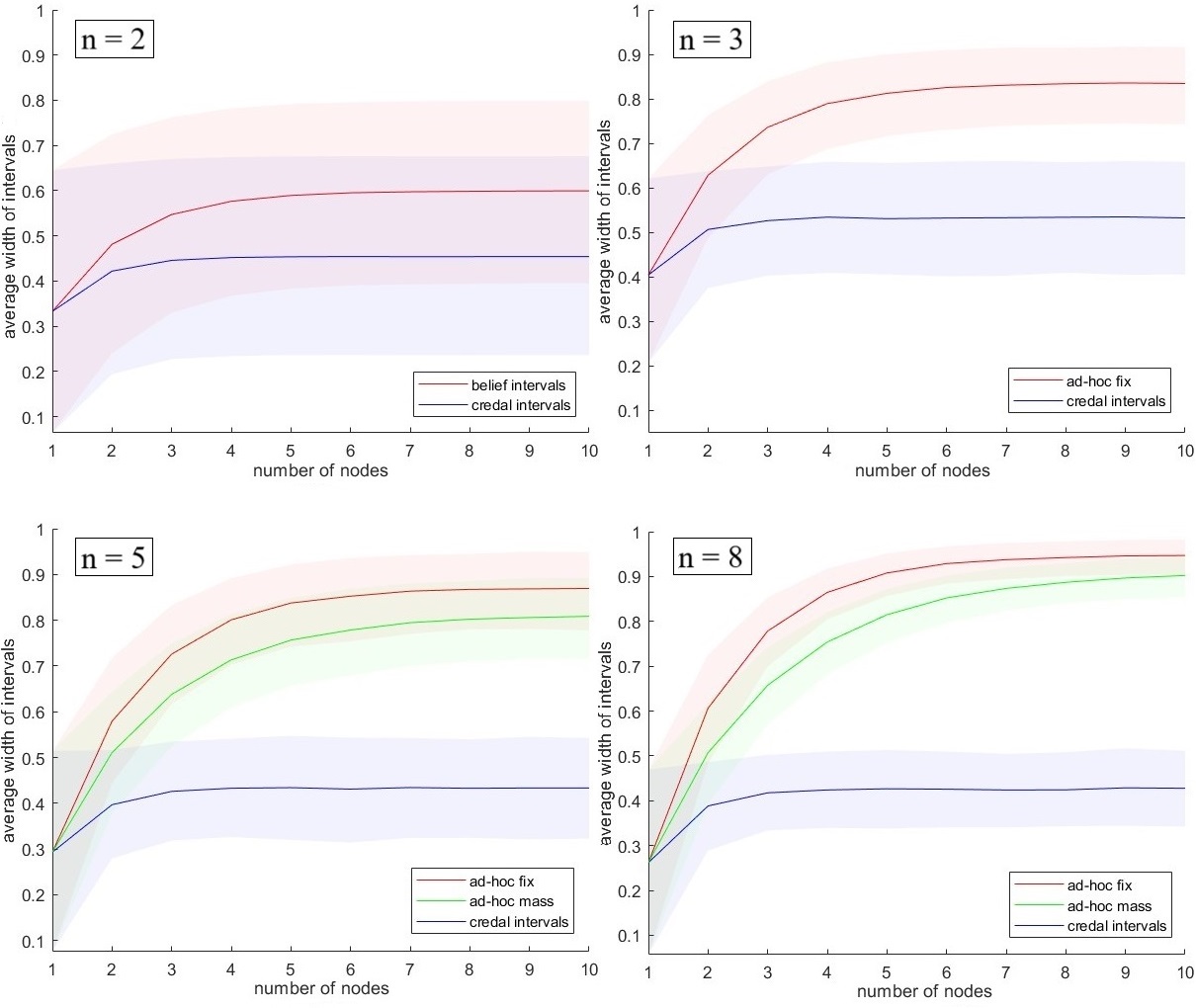}
    \caption{\textbf{Evolution of the average width of intervals} as inference progresses along the chain $X_1\to \dots \to X_{10}$, with each graph corresponding to a different cardinality $n$ of random variables. The colours represent different inference methods: blue for credal intervals, red for belief inference with standard good masses and ad-hoc fixing, green for belief inference with ad-hoc mass functions. Each coloured region contains the width values of 75\% of inferenced intervals. }
    \label{fig:goodsetschain}
\end{figure}

Figure~\ref{fig:goodsetschain} illustrates the average width of intervals as the inference process, using coherent probability intervals sampled uniformly at random, advances along the chain. Let us first note that when $n=2$ or $n=3$ the two belief inference methods we presented coincide, as the mass function associated to a probability interval is then unique (if it exists).
We observe that the process of belief inference becomes increasingly uninformative as the cardinality of the random variables grows. Already with $n=5$ belief inference generates intervals that are twice as large as their credal counterparts, leading to a significant loss of precision.

On the other hand, in the binary case ($|X_i|=2$ for all random variables), we observe that intervals generated using belief inference are asymptotically about $40\%$ larger than their credal counterparts. Moreover, after a single step of inference the average enlargement is approximately $0.0592$; equivalent to $19\%$ of the credal interval, as shown in Table~\ref{table:enlarge}. These observations suggest that belief inference may be effectively applied to more complex binary credal networks with minimal information loss.
\begin{table}[ht]
\centering
\begin{tabular}{@{} l *{6}{c} @{}}
\toprule
\textbf{n} & 2 & 3 & 4 & 5 & 6 & 8 \\ 
\midrule
\small RIW ad-hoc fix 
& 19\% & 26\% & 39\% & 49\% & 55\% & 60\% \\ 
\addlinespace
\small RIW ad-hoc mass 
& 19\% & 26\% & 29\% & 30\% & 31\% & 32\% \\
\bottomrule
\end{tabular}
\caption{\textbf{Approximate Relative Increase in Width (RIW)} after one step of inference in the CN $A\to B$, with $|A|=|B|=n$, comparing intervals generated using belief inference and credal inference.}
\label{table:enlarge}
\end{table}

It is worth noting that the average width of a coherent probability interval can be quite large, meaning that sampling them uniformly at random may not accurately represent real-world scenarios where intervals tend to be narrower. To address this limitation, the following section examines the performance of belief inference when restricted to smaller intervals.

\subsection{Using small intervals}
This section focuses on the analysis of belief inference using conditional and prior probability intervals with limited width. 

We empirically observe that this setup, when using standard good masses, generates a higher number of bad probability intervals as we advance along the chain, thus requiring fixing in order to proceed with the inference process. 
This leads to larger intervals, as we outer-approximate a bad probability interval with a good one, requiring additional steps to reach asymptotic behaviour.
This phenomenon is illustrated in Figure~\ref{fig:comparisonuptoeps}.
\begin{figure}[ht]
    \centering
\includegraphics[width=0.9\linewidth]{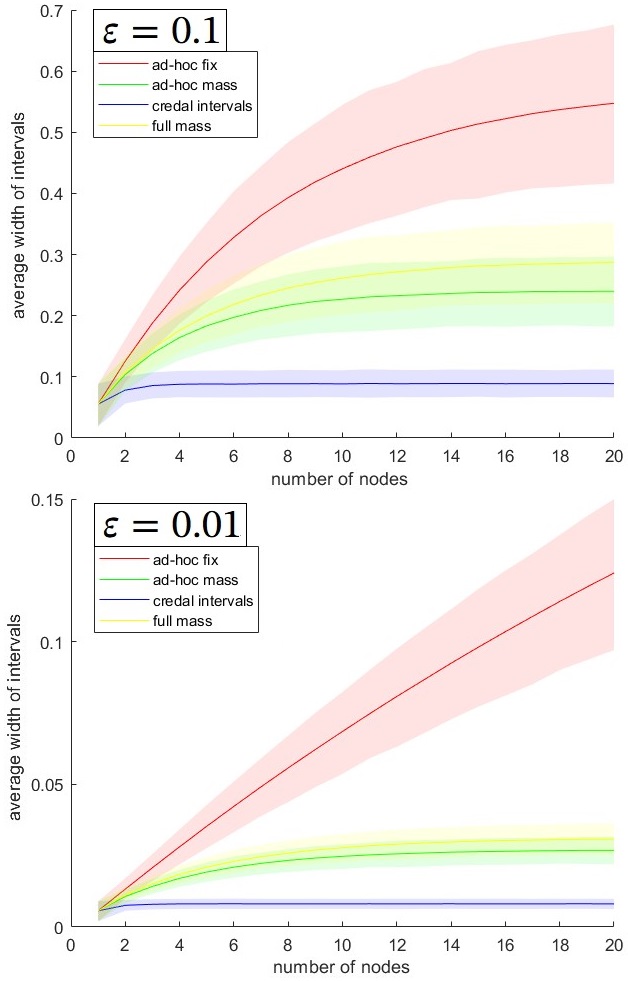}
    \caption{\textbf{Evolution of the average width of intervals} as we make inference along the chain $X_1\to \dots \to X_{20}$, where $|X_i|=5$ for all $i$ and we vary the value of $\varepsilon$, the largest possible width of a conditional interval. Each coloured region contains the width values of 75\% of inferenced intervals.}
    \label{fig:comparisonuptoeps}
\end{figure}
Furthermore, in this setup it is highly likely that belief inference with ad-hoc mass function generates probability intervals that cannot be represented by an ad-hoc mass function, as mentioned in Section \ref{par:fixingintervals}. In our experiments we addressed this issue by uniformly enlarging the probability interval until becoming representable with a mass function. It has been  empirically observed  that this enlargement is often minimal and not impactful in the overall performance.

Looking at Figure~\ref{fig:comparisonuptoeps}, we observe that the ad-hoc fixing method performs significantly worse than other analysed inference approaches. The need to enlarge intervals to proceed with inference lead to excessively wide intervals and it is a great limitation of this method. 

On the other hand, we observe that belief inference with ad-hoc mass functions ensures that the asymptotic width of intervals almost never exceeds $3\varepsilon$, keeping intervals narrow and representative of the inherent uncertainty.  This result suggest that this approach might be extended to more complex graphical models without a significant loss of information.

Inferencing keeping the full mass function without converting it into a probability interval has shown average worse performance than the ad-hoc mass approach.
Most importantly, this method does not guarantee the bounds stated in Theorem \ref{th:bound}. In other words, probability intervals produced by this approach lack any inherent connection to credal intervals, which prevents us from assigning them a meaningful interpretation.
\section{Conclusion}
This paper explored belief inference in credal networks through Dempster-Shafer theory, proposing a framework for uncertainty propagation. Key contributions include the formalization of belief-based methods, novel computational techniques for marginal lower and upper probabilities, and theoretical bounds with standard credal inference. While the theoretical framework guarantees conservative intervals, numerical results reveal significant limitations. Belief inference methods often yield overly conservative intervals, especially in scenarios with higher cardinality.  
 Belief inference in binary chains and the ad-hoc mass function approach appear to be the only frameworks that reliably ensure a close outer-approximation of credal intervals in chains.
It remains open whether this will carry on to more complex graphical models, where achieving precise approximations is often more challenging, in view of the growth in complexity for algorithms on non-binary credal networks.
\appendix
\section{Appendix}
\begin{proof}[Proof of Theorem~\ref{th:inferenceintervals}]
Let us first observe that focal sets of $m_{B|A}$ have the form of $\bigcup_{i=1}^n (a_i\times T^i)$ where $T^i\in 2^{E_B}$ for all $i=1,\dots,n$. It holds that
    \begin{equation*}
        m_{B|A}\left(\bigcup_{i=1}^n (a_i\times T^i)\right)=\prod_{i=1}^n m_{B_{a_i}}(T^i).
    \end{equation*}
    It follows that focal sets of $m_{A,B}$ have the form $\bigcup_{i:a_i\in V} (a_i\times T^i)$ for all $V\in 2^{E_A}$ and $T^i\in 2^{E_B}$. Recalling that $m_{A,B}=m_A\otimes \left(\bigotimes_i m_{B|a_i}\right)$ and applying Dempster's rule of combination, we obtain
    \begin{align*}
       & m_{A,B}\left(\bigcup_{i:a_i\in V} (a_i\times T^i)\right)=\\
       &=m_A(V)\sum_{h:a_h\notin V} \left(\sum_{T^h\in 2^{E_B}} m_{B|A}\left(\bigcup_{i=1}^n (a_i\times T^i)\right)\right)\\
       &=m_A(V)\sum_{h:a_h\notin V} \left(\sum_{T^h\in 2^{E_B}} \prod_{i=1}^n m_{B_{a_i}}(T^i) \right) \\
        &=m_A(V)\prod_{i:a_i\in V}m_{B_{a_i}}(T^i),
    \end{align*}
    where the last equality follows from the fact that $\sum_{T^h\in 2^{E_B}}m_{B_{a_h}}(T^h)=1$ which implies that indices $h:a_h\notin V$ disappear from the product.
   Expression (\ref{eq:lowboundB}) immediately follows from the definitions, and we get
    \begin{align*}
        \underline{p}^B_j&=Bel_{m_B}(\{b_j\})=m_B(\{b_j\})=\\
        &=\sum_{V\in 2^{E_A}}m_{A,B}(V\times \{b_j\})\\
        &=\sum_{V\in 2^{E_A}}  m_A(V)\prod_{i: a_i\in V} m_{B_{a_i}}(\{b_j\}) \\
        &= \sum\limits_{V\in 2^{E_A}}  m_A(V)\prod_{i: a_i\in V} \underline{p}^i_j .
    \end{align*}
The proof for (\ref{eq:upboundB}) is similar: the key is to prove that
\begin{equation*}
    Bel_{m_B}(E_B\smallsetminus\{b_j\})=\sum\limits_{V\in 2^{E_A}}  m_A(V)\prod_{i: a_i\in V} (1-\overline{p}^i_j) 
\end{equation*}
which follows from $Bel_{m_{B_{a_i}}}(E_B\smallsetminus\{b_j\})=1-\overline{p}^i_j$.
\end{proof}
\begin{proof}[Proof of Theorem~\ref{th:mobius}]
    The fact that $m_A^M(\{a_i\})=\underline{p}^A_i$ for all $i=1,\dots, n$ follows immediately from the condition of coherency \hyperref[condcoh2]{[Coh2]}.
    To prove that every subset with cardinality $2\le k \le n-2$ has zero mass we use strong induction on $k$. If $k=2$, then
    \begin{align*}
        m_A^M(\{a_i,a_j\})&=Bel_{m_A^M}(\{a_i,a_j\})-m_A^M(\{a_i\})-m_A^M(\{a_j\})\\
        &=\max \{\underline{p}^A_i+\underline{p}^A_j,1-\sum_{h\ne i,j}\overline{p}^A_h\}-\underline{p}^A_i-\underline{p}^A_j.
    \end{align*}
    Hence, it is enough to prove that $\underline{p}^A_i+\underline{p}^A_j\ge 1-\sum_{h\ne i,j}\overline{p}^A_h$,
    which is the same as showing that
    \begin{align}
        &\overline{S}_A-1-(\overline{p}^A_i-\underline{p}^A_i)-(\overline{p}^A_j-\underline{p}^A_j)\ge 0\nonumber\\
        &\Leftarrow (n-2)(1-\underline{S}_A)-(\overline{p}^A_i-\underline{p}^A_i)-(\overline{p}^A_j-\underline{p}^A_j)\ge 0 \label{salame}
    \end{align}
    where the implication follows from the fact that the probability interval is good thus $\Delta_A\ge 0$. To prove (\ref{salame}) we just note that $n\ge 4$ and
    \begin{equation} \label{condoch3_rewritten}
        \overline{p}^A_h-\underline{p}^A_h\le 1-\underline{S}_A\quad\text{for all $h=1,\dots,n$}
    \end{equation}
    which is just a rewriting of the condition of coherency \hyperref[condcoh3]{[Coh3]}. Therefore, due to the arbitrariness of $i$ and $j$, every set of cardinality $2$ has zero mass.

    Consider $V\in 2^{E_A}$ such that $2<|V|\le n-2$. From the definition of belief we have that
    \begin{align}
        m_A^M(V)&=Bel_{m_A^M}(V)-\sum_{W\subset V}m_A^M(W) \nonumber\\
        &=Bel_{m_A^M}(V)-\sum_{i:a_i\in V}m_A^M(\{a_i\})\label{salame2}\\
        &=\max\{\sum_{i:a_i\in V}\underline{p}_i^A,1-\sum_{i:a_i\notin V}\overline{p}_i^A\}-\sum_{i:a_i\in V}\underline{p}_i^A \nonumber
    \end{align}
where (\ref{salame2}) follows from the induction hypothesis. Therefore, as in the previous case, it suffices to prove that
\begin{equation*}
    \sum_{i:a_i\in V}\underline{p}_i^A\ge 1-\sum_{i:a_i\notin V}\overline{p}_i^A.
\end{equation*}
This inequality is implied by 
\begin{equation*}
    (n-2)(1-\underline{S}_A)-\sum_{i:a_i\in V}(\overline{p}^A_i-\underline{p}^A_i)\ge 0
\end{equation*}
due to the fact that $(\overline{p}^A,\underline{p}^A)$ is a good probability interval. This last inequality follows from (\ref{condoch3_rewritten}) and the fact that $|V|\le n-2$. Therefore, due to the arbitrariness of $V$, every set of cardinality $2< k \le n-2$ has mass $m_A^M$ equal to zero.

To complete the proof, it remains to show that the mass of sets with cardinality \( n-1 \) is equivalent to the standard good mass of these same sets. Using the definition of belief, we have
\begin{align}
    m_A^M(E_A\smallsetminus\{a_i\})&=Bel_{m_A^M}(E_A\smallsetminus\{a_i\})-\;\sum_{\mathclap{W\subset E_A\smallsetminus\{a_i\}}}m_A^M(W)\nonumber\\
    &=\max\{\sum_{h\ne i}\underline{p}^A_h,1-\overline{p}^A_i\}-\sum_{h\ne i}m_A^M(\{a_h\})\label{salame3}\\
    &=1-\overline{p}^A_i -\sum_{h\ne i}\underline{p}_h^A \label{salame4}
\end{align}
where (\ref{salame3}) follows from the fact that every other subset has zero mass, and (\ref{salame4}) follows from the condition of coherency \hyperref[condcoh3]{[Coh3]}.

The fact that $m_A^M(E_A)=\Delta_A$ follows from the fact that $\sum_{V\in 2^{E_A}}m_A^M(V)=1$.
\begin{align*}
    m_A^M(E_A)&=1-\sum_{i=1}^n m_A^M(\{a_i\})-\sum_{i=1}^n m_A^M(E_A\smallsetminus\{a_i\})\\
    &=\overline{S}_A+(n-2)\underline{S}_A-(n-1).
\end{align*}
\end{proof}

\additionalinfo
\begin{acknowledgements}
This work has been partly supported by the PersOn project
(P21-03), which has received funding from Nederlandse Organisatie
voor Wetenschappelijk Onderzoek (NWO). Moreover, we are grateful to four anonymous reviewers for their careful reading and thoughtful comments. Their input was essential in helping us strengthen the paper and present our ideas more clearly.
\end{acknowledgements}

\printbibliography

\end{document}